\documentclass[twocolumn,twoside]{IEEEtran}
\usepackage[dvipsnames]{xcolor}
\usepackage{pdflscape}

\usepackage{colortbl}

\usepackage[latin1]{inputenc}%%%
\usepackage{amssymb}
\usepackage{amsmath}%,amsxtra}
\usepackage[]{graphicx}
\usepackage{dsfont}

\newcommand{\ie}{\emph{i.e.}}
\newcommand{\eg}{\emph{e.g.}}

\newtheorem{theorem}{Theorem}%[section]
\newtheorem{lemma}{Lemma}[section]

\newcommand{\argmin}[1]{\mathop{\arg\!\min}_{#1}}

\usepackage{accentbx}%accents}
\newcommand{\dic}[1]{\textrm{\upaccent{\aboxshift{\char"12}}{$#1$}}}

\newcommand{\calC}{\cp{C}}
\newcommand{\calR}{\cp{R}}
\newcommand{\ps}[2]{\langle{#1},{#2}\rangle}
\newcommand{\psH}[2]{\langle{#1},{#2}\rangle_{\H}}

\newcommand{\normH}[1]{\|{#1}\|_\H}
\newcommand{\normHbig}[1]{\big\|{#1}\big\|_\H}
\newcommand{\normHBig}[1]{\Big\|{#1}\Big\|_\H}

\graphicspath{{matlab/}}

\def\R{\ensuremath{\mathds{R}}}

\def\X{\ensuremath{\mathds{X}}} % input space
\def\Y{\ensuremath{\mathds{Y}}} % output space
\def\H{\ensuremath{\mathds{H}}} %Hilbert
\def\D{\ensuremath{\mathcal{D}}} %dictionary
\newcommand{\cb}[1]{{\ifmmode {\boldsymbol{#1}}\else ${\boldsymbol{#1}}$\fi}}
\newcommand{\cp}[1]{\ifmmode {\mathcal{#1}}\else ${\mathcal{#1}}$\fi}

\newcommand{\dx}{\dic{\bx}}
\newcommand{\dkappa}{\dic{\bkappa}}
\newcommand{\dK}{\dic{\bK}}

\newcommand{\bx}{\cb{x}}
\newcommand{\bxi}{\cb{\xi}}
\newcommand{\balpha}{\cb{\alpha}}
\newcommand{\bbeta}{\cb{\beta}}
\newcommand{\bK}{{\cb{K}}}%_{\!\!\not\,\mathrm{c}}}}
\newcommand{\bI}{{\bf I}}

\newcommand{\bkappa}{\cb{\kappa}}

\begin{document}

%\title{Centering data %Bridging the gap between centered and uncentered data 
%in machine learning: \\an eigenanalysis}% of the eigendecomposition problem}%:  between the eigenpairs of the Gram matrices}

\title{Approximation errors of online sparsification criteria %\\for online learning
}% with kernels}% : when dictionary learning meets feature selection}
%\title{An eigenanalysis of the impact of \\centering data in machine learning}

%\author{Paul Honeine%,~\IEEEmembership{Member,~IEEE}
%\IEEEcompsocitemizethanks{\IEEEcompsocthanksitem 
%\thanks{M. Honeine is with the Institut Charles Delaunay (CNRS), Université de Technologie de Troyes, Troyes, France.}
%}%
\author{Paul~Honeine,~\IEEEmembership{Member~IEEE}
\thanks{P.~Honeine is with the Institut Charles Delaunay (CNRS), Universit\'e de technologie de Troyes, 10000, Troyes, France. Phone: +33(0)325715625; Fax: +33(0)325715699; E-mail: paul.honeine@utt.fr
}}

\markboth{}%IEEE Trans. on Information Theory,~Vol.~XX, No.~XX,~XX~201X}
{Honeine: Approximation errors of online sparsification criteria}

%\begin{document}

%\editor{~\\
%~\\
%~\\
%\mbox{~}\hfill ``All happy families are alike; each unhappy 
%\\\mbox{~}\hfill  family is unhappy in its own way.''
%\\\mbox{~}\hfill  Leo Tolstoy}%: Anna Karenina (1878)}

\sloppy

\maketitle

%\IEEEcompsoctitleabstractindextext{
\begin{abstract}

Many machine learning frameworks, such as resource-allocating networks, kernel-based methods, Gaussian processes, and radial-basis-function networks, require a sparsification scheme in order to address the online learning paradigm. For this purpose, several online sparsification criteria have been proposed to restrict the model definition on a subset of samples. The most known criterion is the (linear) approximation criterion, which discards any sample that can be well represented by the already contributing samples, an operation with excessive computational complexity. Several computationally efficient sparsification criteria have been introduced in the literature, such as the distance, the coherence and the Babel criteria. In this paper, we provide a framework that connects these sparsification criteria to the issue of approximating samples, by deriving theoretical bounds on the approximation errors. Moreover, we investigate the error of approximating any feature, by proposing upper-bounds on the approximation error for each of the aforementioned sparsification criteria. Two classes of features are described in detail, the empirical mean and the principal axes in the kernel principal component analysis.
\end{abstract}

\begin{keywords} 
Sparse approximation, adaptive filtering, kernel-based methods, resource-allocating networks, Gaussian processes, Gram matrix, machine learning, pattern recognition, online learning, sparsification criteria.
\end{keywords}

%
%\makeatletter
%\if@draftclsmode
%\begin{center}
%\bfseries EDICS Category: SSP-SSEP
%\end{center}
%\fi
%\makeatother
\IEEEpeerreviewmaketitle

%%\section*{Table of contents}
%\newpage
%
%\tableofcontents
%
%\newpage

\section{Introduction}

\PARstart{D}{ata deluge} in the era of ``Big Data'' brings new challenges (and opportunities) in the area of machine learning and signal processing \cite{stsp.bigdata14,tsp.bigdata14,spm.bigdata14}. Demanding online learning, this paradigm cannot be addressed directly by most (if not all) conventional learning machines, such as resource-allocating networks \cite{Platt91}, kernel-based methods for classification and regression \cite{Vap98}, Gaussian processes \cite{gpml}, radial-basis-function networks \cite{Huang05ageneralized} and kernel principal component analysis \cite{12.tpami}, only to name a few. Indeed, these machines share essentially the same underlying model, with as many parameters to be estimated as training samples, as defined by the ``Representer Theorem'' \cite{Representer}. This model is inappropriate in online learning, where a new sample is available at each instant. To stay computationally tractable, one needs to restrict the incrementation in the model complexity, by selecting the subset of samples that contributes to a reduced-order model as an approximation of the full-order feature to be estimated.%, thus approximating the optimal feature with a sparse model. 

In order to overcome this bottleneck in online learning, sparsification schemes have been proposed for all the aforementioned machines, defined as follows: %determining %, yielding a sparse approximation of the full-order model. Indeed, in an online setting, 
%these contributing samples are called atoms and are collected in a set called dictionary. 
%To this end, a sparsification criterion is explored in order to determine, 
at each instant, it determines if the new sample can be safely discarded from contributing to the order growth of the model; otherwise, the sample needs to take part in the order incrementation. The most known online sparsification criteria is the approximation criterion, also called approximate linear dependency. It has been widely investigated in the literature, for Gaussian processes \cite{Csato02}, kernel recursive least squares algorithm \cite{Eng04},  kernel least mean square algorithm \cite{Pokharel2009}, and kernel principal component analysis \cite{12.tpami}. This criterion determines the relevance of discarding or accepting the current sample by comparing, to a predefined threshold, the residual error of approximating it with a representation (\ie, linear combination) of samples --- or nonlinearly mapped samples as in kernel methods --- already contributing to the model. A crucial issue in the approximation criterion is its computational complexity, which scales cubically with the model order.

Several computationally efficient sparsification criteria have been introduced in the literature, with essentially the same computational complexity that scales linearly with the model order. These sparsification criteria rely on the topology of the samples in order to select the most relevant samples. The most widely investigated criteria are the distance and the coherence criteria, as well as the Babel criterion. The distance criterion, introduced by Platt in \cite{Platt91} to control the complexity of resource-allocating networks in radial-basis-function networks, retains the most mutually distant samples ; see also \cite{Babu2013,Yang2013} %,Vukovic2013} 
for recent advances. The coherence criterion, introduced by Honeine, Richard, and Bermudez in \cite{Hon07.isit,Ric09.tsp} with the recent advances in compressed sensing \cite{Tro04,Elad2010book}, retains samples that are mutually least coherent. As an extension of the coherence criterion, the Babel criterion uses the cumulative coherence as a measure of diversity \cite{Fan2014}.

These sparsification criteria have been separately investigated in the literature. To the best of our knowledge, there is no work that studies all these sparsification criteria together. The conducted analyses have been often based on the computational complexity, as advocated in \cite{Ric09.tsp,KernelAdaptiveFiltering} by criticizing the computational cost of the approximation criterion in favor of the other sparsification criteria. In \cite{Hon07.isit,Ric09.tsp,12.ssp.one_class}, we have developed with colleagues several theoretical results that allows to compare the coherence to the approximation criterion. These results have not been extended to other sparsification criteria, and were demonstrated for the particular case of unit-norm data.

%Despite its computational cost, the approximation criterion is motivated by some theoretical results, as derived by \cite{Eng04} and \cite{12.tpami}. These results are revisited for the coherence criterion by ...with new properties, making them more suitable and comparable with the higher-computational approximation criterion. 

This paper presents a framework to study online sparsification criteria by cross-fertilizing previously derived results, by obtaining often tighter bounds, and by extending these results to other sparsification criteria such as the distance and the Babel criteria. One the one hand, we bridge the gap between the approximation criterion and the other online sparsification criteria, firstly by providing upper bounds on the error of approximating, with samples already retained, any sample discarded by the sparsification criterion, and secondly by providing lower bounds on the error of approximating accepted samples. On the other hand, we examine the relevance of approximating any feature with a sparse model obtained with any of the aforementioned sparsification criteria, including the approximation criterion. We provide upper bounds on the error of approximating any feature in the general case. Furthermore, we explore in detail two particular features, the empirical mean (\ie, centroid, studied for instance in \cite{12.isit,Jenssen13}) and the principal axes in the kernel principal component analysis (kernel-PCA, \cite{KPCA}). The big picture of the cross-fertilization and extensions given in this paper is illustrated in \tablename~\ref{tab:birdseye}.

The remainder of this paper is organized as follows. Next section introduces the kernel-based machines for online learning and presents the key issues studied in this work. Section~\ref{sec:criteria} presents the aforementioned computationally efficient sparsification criteria. Section~\ref{sec:approx.error} investigates bounds on the error of approximation samples, either discarded or accepted by any sparsification criterion. These results are extended in Section~\ref{sec:feature} to the problem of approximating any feature. Section~\ref{sec:final_remarks} concludes this document with some discussions.

%\begin{landscape}
\begin{table}[t]
%\hspace*{-.75cm}
%\footnotesize
%\scriptsize
\renewcommand{\arraystretch}{1.3} 
\begin{center}
\begin{tabular}{lccccl}
% & \multicolumn{4}{c}{sparsity measures and sparsification criteria} \\ \cline{2-5}
 					& \rotatebox{90}{Distance} & \rotatebox{90}{Approximation} & \rotatebox{90}{Coherence} & \rotatebox{90}{Babel}
					&%\rotatebox{90}{\cf Section}% 
					\!\!\!Section\!\!\!\! 
					\\\hline
%\rowcolor[gray]{.97} 
Reference: most known work & \cite{Platt91} & \cite{Csato02} & \cite{Ric09.tsp}%,Khandan2013} 
& \cite{Tro04} \\
Reference: more recent work & \cite{KernelAdaptiveFiltering} & \cite{12.tpami} & \cite{13.spl.dictionary} & \cite{Fan2014} \\
\rowcolor[gray]{.97} \bf Approximation of any sample & \bf $\checkmark$ & \bf $\cdot$ & \bf $\checkmark$ & \bf $\checkmark$  & \bf \ref{sec:approx.error} \\
~\rotatebox[origin=c]{180}{$\large\Lsh$} Error on discarded samples & $\checkmark$ & $\cdot$ & \cite{Hon07.isit} & $\checkmark$ & \ref{sec:approx.error.discard} \\
 ~\rotatebox[origin=c]{180}{$\large\Lsh$} Error on any atom & $\checkmark$ & $\cdot$ & \cite{Ric09.tsp} & \color{gray} \cite{Fan2014} & \ref{sec:approx.atom} \\
\rowcolor[gray]{.97} \bf Approximation of any feature & \bf $\checkmark$ & \bf $\checkmark$ & \bf $\checkmark$ & \bf $\checkmark$ & \bf \ref{sec:feature} \\
%\rowcolor[gray]{.97}
~\rotatebox[origin=c]{180}{$\large\Lsh$} Error on the mean (centroid) & $\checkmark$ & $\checkmark$ & \cite{12.ssp.one_class} & $\checkmark$ & \ref{sec:feature.mean} \\
~\rotatebox[origin=c]{180}{$\large\Lsh$} Error on the principal axes & $\checkmark$ & \color{gray} \cite{Eng04}%,12.tpami} %\cite{} 
& \color{gray} \cite{Hon07.isit} & $\checkmark$ & \ref{sec:feature.kpca} \\
\hline
\end{tabular} 
\end{center}
\caption{A birds eye view of this paper. Some of the results were previously studied for unit-norm kernels, as shown with the references given in the table (where $\cdot$ denotes triviality). In this work, we provide an extensive study that completes the analysis to all sparsification criteria, often with tighter bounds ({\color{gray} shown in gray color}), and we derive new theoretical results. Moreover, we generalize these results to any type of kernel, beyond unit-norm kernels.}
\label{tab:birdseye}
\end{table}
%\end{landscape}

\section{Kernel-based machines for online learning}

In this section, we introduce the kernel-based machines for online learning, by presenting the approximation criterion with the key issues studied in this paper.

%%%

\subsection{Machine learning and online learning}

Machine learning seeks a feature $\psi(\cdot)$ connecting an input space $\X \subset \R^d$ to an output space $\Y \subset \R$, by using a set of training samples, denoted $\{(\bx_1,y_1), (\bx_2,y_2), \ldots , (\bx_n,y_n)\}$ with $(\bx_k,y_k) \in \X\times\Y$. Considering a loss function $\calC(\cdot,\cdot)$ defined on $\Y \times \Y$ that measures the error between the desired output and the estimated one with $\psi(\cdot)$, the optimization problem consists in minimizing a regularized empirical risk of the form
\begin{equation}\label{eq:risk}
    \argmin{\psi(\cdot) \in \H} %\frac{1}{n} 
    \sum_{i=1}^n \calC(\psi(\bx_i),y_i) + \eta \, \calR (\|\psi(\cdot)\|_\H^2),
\end{equation}
where $\H$ is the feature space of candidate solutions and $\eta$ is a parameter that controls the tradeoff between the fitness error (first term) and the regularity of the solution (second term) with $\calR(\cdot)$ being a monotonically increasing function. Examples of loss functions are the quadratic loss $| \psi(\bx_i) - y_i |^2$, the hinge loss $( 1 - \psi(\bx_i) y_i )_+$ of the SVM \cite{Vap98}, the logistic regression $\log( 1 + \exp(-\psi(\bx_i) y_i))$, as well as the unsupervised loss function $- |\psi(\bx_i)|^2$ which is related to the PCA.

Let $\kappa\!:\X \times \X\rightarrow\R$ be a positive definite kernel, and $(\H,\psH{\cdot}{\cdot})$ the induced reproducing kernel Hilbert space (RKHS) with its inner product. The reproducing property states that any function $\psi(\cdot)$ of $\H$ can be evaluated at any sample $\bx_i$ of $\X$ using $\psi(\bx_i) = \psH{\psi(\cdot)}{\kappa(\cdot,\bx_i)}$. This property shows that any sample $\bx_i$ of $\X$ is represented with $\kappa(\cdot,\bx_i)$ in the space $\H$. Moreover, the reproducing property leads to the so-called kernel trick, that is for any pair of samples $(\bx_i,\bx_j)$, we have $\psH{\kappa(\cdot,\bx_i)}{\kappa(\cdot,\bx_j)} = \kappa(\bx_i,\bx_j)$. In particular, $\normH{\kappa(\cdot,\bx_i)} = \psH{\kappa(\cdot,\bx_i)}{\kappa(\cdot,\bx_i)} = \kappa(\bx_i,\bx_i)$. The most used kernels and there expressions are as follows:
\begin{center}
\renewcommand{\arraystretch}{1.1} 
\begin{tabular}{l@{\qquad}c}
Kernel & $\kappa(\bx_i,\bx_j)$  \\\hline\hline
Linear & $\ps{\bx_i}{\bx_j}$ \\
Polynomial & $\left(\ps{\bx_i}{\bx_j} + c\right)^p$\\
Exponential & $\exp\left(\ps{\bx_i}{\bx_j}\right)$ \\
Gaussian & $\exp\left(\frac{-1}{2\sigma^2} \|\bx_i - \bx_j\|^2\right)$ \\
\hline\hline
\end{tabular}
\end{center}
Among these kernels, only the Gaussian kernel is unit-norm, that is $\normH{\kappa(\bx,\cdot)}=1$ for any sample $\bx \in \X$. Other kernels can be unit-norm on some restricted $\X$, such as the linear kernel when dealing with unit-norm samples. In this paper, we do not restrict ourselves to any particular kernel or space $\X$. We denote 
\begin{equation*}
 	{~r^2 = \inf_{\bx \in \X} \kappa(\bx,\bx)~}
		\qquad \text{and} \qquad 
 	{~R^2 = \sup_{\bx \in \X} \kappa(\bx,\bx)~}.
\end{equation*}
For unit-norm kernels, we get $R=r=1$.

The Representer Theorem provides a principal result that is essential in kernel-based machines for classification and regression, as well as unsupervised learning. It states that the solution of the optimization problem \eqref{eq:risk} takes the form
\begin{equation}\label{eq:repr}
		\psi(\cdot) = \sum_{i=1}^{n} \alpha_i \, \kappa(\bx_i,\cdot).
\end{equation}
The proof of this theorem is derived in \cite{Representer}, and a sketch of proof is given in the footnote\footnote{To prove the Representer Theorem \eqref{eq:repr}, we decompose any function $\psi(\cdot)$ of $\H$ into $\psi(\cdot)=\sum_{i=1}^n \alpha_i\,\kappa(\bx_i,\cdot) + \psi^\perp(\cdot)$, where $\langle\psi^\perp(\cdot),\kappa(\bx_i,\cdot)\rangle_{\H}=0$ for all $i=1,2,\ldots,n$. On the one hand, any evaluation $\psi(\bx_i)$ is independent of $\psi^\perp(\cdot)$ since $\psi(\bx_i)=\psH{\psi(\cdot)}{\kappa(\bx_i,\cdot)}$. On the other hand, the monotonically increasing function $\calR(\cdot)$ guarantees that $\calR(\|\psi(\cdot)\|_\H^2) = \calR(\|\sum_{i=1}^n \alpha_i\,\kappa(\bx_i,\cdot)+\psi^\perp(\cdot)\|_\H^2) \geq \calR(\|\sum_{i=1}^n \alpha_i\,\kappa(\bx_i,\cdot)\|_\H^2)$, where the Pythagorean theorem is used. Therefore, a null $\psi^\perp(\cdot)$ minimizes the regularization term without affecting the fitness term.}. This theorem shows that the optimal solution has as many parameters $\alpha_i$ to be estimated as the number of available samples $(\bx_i,y_i)$. This result constitutes the principal bottleneck for online learning. Indeed, in an online setting, the solution should be adapted based on a new sample available at each instant, namely $(\bx_t,y_t)$ at instant $t$. Thus, by including the new pair $(\bx_t,y_t)$ in the training set, the corresponding parameter $\alpha_t$ is be added to the set of parameters to be estimated, by following the Representer Theorem. As a consequence, the order of the model \eqref{eq:repr} is continuously increasing. 

To overcome this bottleneck, one needs to control the growth of the model order at each instant, by keeping only a fraction of the kernel functions in the expansion \eqref{eq:repr}. The reduced-order model takes the form
\begin{equation}\label{eq:repr_m}
       \psi(\cdot) = \sum_{j=1}^{m} \alpha_{j} \, \kappa(\dx_j,\cdot)
\end{equation}
with $m \ll t$, predefined or dependent on $t$. In this expression, $\{\dx_1, \dx_2, \ldots, \dx_m\}$ is a subset of $\{\bx_1, \bx_2, \ldots, \bx_t\}$, namely $\dx_j$ is some $\bx_{\omega_j}$ with $\omega_j \in \{1, 2, \ldots, t\}$. We denote by dictionary the set $\D=\{\kappa(\dx_1,\cdot), \kappa(\dx_2,\cdot), \ldots, \kappa(\dx_m,\cdot)\}$, and by atoms its elements. Throughout this paper, all quantities associated to the dictionary have an accent (by analogy to phonetics, where stress accents are associated to prominence). This is the case for instance of the $m$-by-$m$ Gram matrix $\dK$ whose $(i,j)$-th entry is $\kappa(\dx_i,\dx_j)$. The eigenvalues of this matrix are denoted $\dic{\lambda}_{1}, \dic{\lambda}_{2}, \ldots, \dic{\lambda}_{m}$, given in non-increasing order.

The optimization problem is two-fold at each instant: selecting the proper dictionary $\D=\{\kappa(\dx_1,\cdot), \kappa(\dx_2,\cdot), \ldots, \kappa(\dx_m,\cdot)\}$ and estimating the corresponding parameters $\alpha_1, \alpha_2, \ldots, \alpha_m$. New challenges (and opportunities) arise in an online learning setting. Determining the optimal dictionary at each instant is a combinatorial optimization problem, when optimality is measured by comparing reduced-order solution \eqref{eq:repr_m} to the feature in its full-order form \eqref{eq:repr}. An elegant way to overcome this computationally intractable problem, is a recursive update, by determining if the new kernel function $\kappa(\bx_t,\cdot)$ needs to be included to the dictionary, or it can be discarded since it is efficiently approximated with atoms already belonging to the dictionary. This is the essence of the approximation criterion.

%et aussi Sparse Gaussian Process Regression :
%2862-a-matching-pursuit-approach-to-sparse-gaussian-process-regression.pdf
%notamment Section "3 Selection of basis functions"

%$m$ is the size (\ie, cardinality) of the dictionary $\D$.

\subsection{Approximation criterion}\label{sec:approx}

The (linear) approximation criterion was initially proposed in \cite{Bau01} for classification and regression, and in \cite{Csato01} for Gaussian processes. In online learning with kernels, as studied for system identification in \cite{Eng04} and more recently for kernel principal component analysis in \cite{12.tpami}, it operates as follows: the current sample is discarded (not included in the dictionary), if it can be sufficiently represented by a linear combination of atoms already belonging to the dictionary; otherwise, it is included in the dictionary. Formally, %This is done by comparing the candidate kernel function with its projection onto the subspace spanned by the atoms of the dictionary. 
the kernel function $\kappa(\bx_t,\cdot)$ is included in the dictionary if %(and only if)
\begin{equation}\label{eq:approx.crit}
    \min_{\xi_1\cdots\xi_m}\normHBig{\kappa(\bx_t,\cdot) - \sum_{j=1}^m \xi_j\,\kappa(\dx_j,\cdot)}^2 \geq \delta^2,
\end{equation}
where $\delta$ is a positive threshold parameter that controls the level of sparseness. The above norm is the residual error obtained by projecting $\kappa(\bx_t,\cdot)$ onto the space spanned by the dictionary. The optimal value of each coefficient $\xi_j$ is obtained by nullifying the derivative of the above cost function with respect to it, which leads to
\begin{equation*}%\label{eq:approx.proj}
    \bxi = \dK^{-1} \dkappa(\dx_t),
\end{equation*}
where $\dkappa(\bx_t)$ is the column vector of entries $\kappa(\dx_j,\bx_t)$, for $j=1,2,\ldots, m$. By injecting this expression in the condition \eqref{eq:approx.crit}, we get the following condition of accepting the current kernel function: %, written in matrix form
\begin{equation}\label{eq:approx.crit1}
	\kappa(\bx_t,\bx_t) - \dkappa(\bx_t)^\top \dK^{-1} \dkappa(\bx_t) \geq \delta^2.
\end{equation}
The resulting dictionary is called $\delta$-approximate, satisfying the relation 
\begin{equation*}
    \min_{i=1\cdots m} \min_{\xi_1\cdots \xi_m}\normHBig{\kappa(\dx_i,\cdot) - \mathop{\sum_{j=1}^m}_{j \neq i} \xi_j\,\kappa(\dx_j,\cdot)} \geq \delta
    .%,
\end{equation*}
One could also include a removal process, in the same spirit as the fixed-budget concept, by discarding the atom that can be well approximated with the other atoms, as investigated for instance in \cite{NguyenTuong11}. Nonetheless, the dictionary is still $\delta$-approximate. The use of a removal process does not affect the results given in this paper.

\subsection{Issues studied in this paper}

In the following, we describe several issues that motivates (and structures) this work, illustrated here with respect to the approximation criterion.

\subsection*{Computational complexity}

The approximation criterion requires the inversion of the Gram matrix associated to the dictionary, which is the most computational expensive process. Its computational complexity scales cubically with the size of the dictionary, \ie, $\cp{O}(m^3)$ operations. %One can reduce such computational cost by using a recursive rule when an element is included in the dictionary. 
Moreover, the evaluation of the condition expressed in \eqref{eq:approx.crit1} requires two matrix multiplications at each instant. These computation cost may counteract the benefits of several online learning techniques, such as gradient-based and least-mean-square algorithms (\eg, LMS, NLMS, affine projection, ...). 

To reduce the computational burden of the approximation criterion, several computationally efficient sparsification criteria have been proposed in the literature, sharing essentially the same computational complexity that scales linearly with the size of the dictionary, \ie, $\cp{O}(m)$ operations at each instant. The most known criteria are the distance, the coherence and the Babel criteria; see Section~\ref{sec:criteria} for a description. 

\subsection*{Approximation error of any sample}

The approximation criterion relies on establishing a dictionary such that the error of approximating each of its atoms, with a linear combination of the other atoms, cannot be smaller than the given threshold $\delta$. Moreover, the decision of discarding any sample from the dictionary is defined by the same process, namely when its approximation error, with a linear combination of the other atoms, is smaller than the same threshold $\delta$. While the approximation criterion possesses such duality between accepting and discarding samples at the very same value of thresholding the approximation error, this is not the case of the other sparsification criteria. 

In Section~\ref{sec:approx.error}, we bridge the gap between the approximation criterion and the other online sparsification criteria. For this purpose, on the one hand in Section~\ref{sec:approx.error.discard}, we derive upper bounds on the error of approximating a discarded samples with atoms of a dictionary obtained by the distance, the coherence, or the Babel criterion. One the other hand in Section~\ref{sec:approx.atom}, we derive lower bounds on the error of approximating any atom with the other atoms of the sparse dictionary under scrutiny.

\subsection*{From approximating samples to approximating features}

All the aforementioned sparsification criteria operate in a pre-processing scheme, by selecting samples independently of the resulting sparse representation of the feature. In other words, the selection of the relevant subset $\{\dx_1, \dx_2, \ldots, \dx_m\}$ from the set $\{\bx_1, \bx_2, \ldots, \bx_t\}$ is only based on the topology of the samples; it is independent of the power of the dictionary to approximate accurately any feature of the form \eqref{eq:repr} with the reduced-order model \eqref{eq:repr_m}.

In Section~\ref{sec:feature}, we study the relevance of approximating any feature with a sparse dictionary obtained by any sparsification criterion, including the approximation criterion. We derive upper bounds on the approximation error of any feature, before examining in detail two particular class of features, the empirical mean studied in Section~\ref{sec:feature.mean} and the most relevant principal axes in kernel-PCA investigated in Section~\ref{sec:feature.kpca}.

%%%%%%%%%%%%%%
\section{Online sparsification criteria}\label{sec:criteria}

With a novel sample $\bx_t$ available at each instant $t$, a sparsification rule determines if $\kappa(\bx_{t}, \cdot)$ should be included in the dictionary, by incrementing the model order $m$ and setting $\dx_{m+1} = \bx_{t}$. The sparsification criteria measure the relevance of such complexity-incrementation by comparing the current kernel function $\kappa(\bx_{t}, \cdot)$ with the atoms of the dictionary. They are defined by either a dissimilarity measure, \ie, constructing the dictionary with the most mutually distant atoms, or a similarity measure, \ie, constructing the dictionary with the least coherent or correlated atoms. To this end, a threshold is used to control the level of sparsity of the dictionary. The most investigated criteria are outlined in the following.%These criteria can be regrouped according to the level on which the measure is conducted, either individually with pairwise rule of the form $\min_{j=1\cdots m} f(\bx_t, \dx_j)$ to be compared to some given threshold, or with all (or a subset of) atoms together, thus with a form $f(\bx_t, \dx_1, \dx_2, \ldots, \dx_m)$ to be compared to some fixed threshold. The value of the threshold determines the level of sparsity of the dictionary.

\subsection{Distance criterion}\label{sec:distance}

It is natural to propose a sparsification criterion that constructs a dictionary with large distances between its entries, thus discarding samples that are too close to any of the atoms already belonging to the dictionary. The current kernel function $\kappa(\bx_t,\cdot)$ is included in the dictionary if %(and only if)
\begin{equation}\label{eq:dist.crit}
     \min_{j=1\cdots m} \min_{\xi} \normH{\kappa(\bx_t,\cdot) - \xi \, \kappa(\dx_j,\cdot)} \geq \delta,
\end{equation}
for a predefined positive threshold $\delta$; otherwise, it can be efficiently approximated, up to a multiplicative constant, with an atom of the dictionary. It is easy to see that the optimal value of the scaling factor $\xi$ is $\kappa(\bx_t,\dx_j) / \kappa(\dx_j,\dx_j)$, since the left-hand-side in the above expression is residual error of the projection of $\kappa(\bx_t,\cdot)$ onto $\kappa(\dx_j,\cdot)$ (in the same spirit as the approximation criterion). This allows to simplify the condition \eqref{eq:dist.crit} to get
\begin{equation}\label{eq:dist.crit1}
    \min_{j=1\cdots m}
    \left(
    \kappa(\bx_t,\bx_t) - \frac{\kappa(\bx_t,\dx_j)^2}{\kappa(\dx_j,\dx_j)} 
    \right) \geq \delta^2.
\end{equation}
%In the following, we describe two variants of the distance criterion.
The resulting dictionary, called $\delta$-distant, satisfies for any pair $(\dx_i, \dx_j)$:
\begin{equation}\label{eq:dist}
    \kappa(\dx_i,\dx_i) - \frac{\kappa(\dx_i,\dx_j)^2}{\kappa(\dx_j,\dx_j)} \geq \delta^2.
\end{equation}
For unit-norm atoms, this expression reduces to the condition $|\kappa(\dx_i,\dx_j)| \leq \sqrt{1-\delta^2}$. This sparsification criterion has been extensively used in the literature under different names, such as the novelty criterion proposed in \cite{Platt91} (where the scaling factor was dropped and a prediction error mechanism was included in a second stage; see also \cite{Rosipal97RAN,Huang05ageneralized}) and the quantized criterion defined in \cite{QuantizedKLMS}.

\subsection{Coherence criterion} \label{sec:coherence}

The coherence measure has been extensively studied in the literature of compressed sensing in the particular case of the linear kernel with unit-norm samples \cite{Tro04,Elad2010book}. In the more general case with the kernel formalism, the coherence of a dictionary is defined with the measure
\begin{equation}\label{eq:coher}
	\mathop{\max_{i,j=1\cdots m}}_{i \neq j} 
	%\mathrm{coh}(\kappa(\dx_i,\cdot) , \kappa(\dx_j,\cdot))
	%\mathrm{coh}(\dx_i,\dx_j) 
	\frac{|{\kappa(\dx_i,\dx_j)}|} {\sqrt{\kappa(\dx_i,\dx_i) \, \kappa(\dx_j,\dx_j)}}
	%\leq \gamma
	,
\end{equation}
which corresponds to the largest value of the cosine angle between any pair of atoms, since the above objective function can be written as
\begin{equation*}
\frac{|\psH{\kappa(\dx_i,\cdot)}{\kappa(\dx_j,\cdot)}|}
        {\normH{\kappa(\dx_i,\cdot)} \normH{\kappa(\dx_j,\cdot)}}.
\end{equation*}
The coherence criterion introduced in \cite{Hon07.isit,Ric09.tsp} constructs a dictionary with atoms that are mutually least coherent, by restricting this measure below some predefined value $\gamma \in \; [ 0 \; ; 1]$, where a null value yields an orthogonal basis. The criterion includes the current kernel function $\kappa(\bx_t,\cdot)$ in the dictionary if
\begin{equation}\label{eq:coher.crit}
	\max_{j=1 \cdots m}
	\frac{|{\kappa(\bx_t,\dx_j)}|} {\sqrt{\kappa(\bx_t,\bx_t) \, \kappa(\dx_j,\dx_j)}}
	\leq \gamma.
\end{equation}
It is worth noting that the denominator in each of the above expressions reduces to 1 when dealing with unit-norm atoms, thus expression \eqref{eq:coher.crit} becomes
\begin{equation*}
	\max_{j=1\cdots m} |\kappa(\bx_t,\dx_j)| \leq \gamma.
\end{equation*}

\subsection{Babel criterion}\label{sec:Babel}

While the coherence measure examines the largest correlation between all pairs of atoms in a dictionary, a more thorough analysis is provided by the Babel measure, which considering the maximum cumulative correlation between an atom and all the atoms of the dictionary \cite{Gil03b,Tro04}. The Babel criterion for online sparsification is defined as follows: the current kernel function $\kappa(\bx_t,\cdot)$ is included in the dictionary if
\begin{equation}\label{eq:babel.crit}
      \sum_{j=1}^m |\kappa(\bx_t, \dx_j)| \leq \gamma,
\end{equation}
for a given positive threshold $\gamma$ \cite{Fan2014}. The resulting dictionary, called $\gamma$-Babel, satisfies
\begin{equation}\label{eq:babelD}
      \max_{i=1\cdots m} \mathop{\sum_{j=1}^m}_{j \neq i} |\kappa(\dx_i, \dx_j)|
      \leq \gamma.
\end{equation}
By analogy with the coherence measure, which corresponds to the $\infty$-norm when dealing with unit-norm atoms, the Babel measure\footnote{One can also consider a normalized version of the Babel measure, by substituting $\kappa(\bx_t, \dx_j)$ in \eqref{eq:babel.crit} with ${{\kappa(\bx_t,\dx_j)}}/{\sqrt{\kappa(\bx_t,\bx_t) \, \kappa(\dx_j,\dx_j)}}$. These two definitions are equivalent when dealing with unit-norm atoms. To the best of our knowledge, this formulation is not used in the literature. Moreover, it looses the matrix-norm notion.} is related to the $1$-norm of the Gram matrix, where $\|\dK\|_1 = \max_i \sum_j |\kappa(\dx_i, \dx_j)|$.

%%%%%%%%%%%%%%%%%%%%%%%%%%%%%%%
%%%%%%%%%%%%%%%%%%%%%%%%%%%%%%%

%%%%%%%%%%%%%%%%%%%%%%%%%%%%%%%
%%%%%%%%%%%%%%%%%%%%%%%%%%%%%%%

\section{Approximation error of any sample}\label{sec:approx.error}

In this section, we study the elementary issue of approximating a sample by the span of a sparse dictionary, in the kernel-based framework. To this end, this issue is considered in its two folds: one the one hand, the error of approximating a discarded sample, and on the other hand, the error of approximating any accepted sample, namely approximating any atom of the dictionary with all the other atoms. We provide upper bounds on the former and lower bounds on the latter, for each of the sparsification criteria studied in previous section. It is worth noting that only the approximation criterion relies on a duality of discarding and accepting samples at the very same value in thresholding the approximation error, which is not the case of the other criteria, as examined in the following.

Let $\dic{\cp{P}}$ be the projection operator onto the subspace spanned by the atoms $\kappa(\dx_1,\cdot), \ldots, \kappa(\dx_m,\cdot)$ of a dictionary resulting from a sparsification criterion. Thus, for any sample $\bx$, the projection of the kernel function $\kappa(\bx,\cdot)$ onto this subspace is given by $\dic{\cp{P}}\kappa(\bx,\cdot)$. The quadratic norm of the latter corresponds to the maximum inner product $\psH{\kappa(\bx,\cdot)}{\varphi(\cdot)}$ over all the unit-norm functions $\varphi(\cdot)$ of that subspace. By writing $\varphi(\cdot) = \sum_{j=1}^m \beta_j \kappa(\dx_j,\cdot) / \normH{\sum_{j=1}^m \beta_j \kappa(\dx_j,\cdot)}$, one gets
\begin{align}\label{eq:projection}
	\normH{\dic{\cp{P}}\kappa(\bx,\cdot)}^2
	&= \max_\bbeta \frac{\psH{\sum_{j=1}^m \beta_j \kappa(\dx_j,\cdot)}{\kappa(\bx,\cdot)}}{\normH{\sum_{j=1}^m \beta_j \kappa(\dx_j,\cdot)}} \nonumber
\\	&= \max_\bbeta \frac{\sum_{j=1}^m \beta_j \kappa(\bx,\dx_j)}{\normH{\sum_{j=1}^m \beta_j \kappa(\dx_j,\cdot)}}.
\end{align}
Moreover, the Pythagorean Theorem allows to measure the residual norm of this projection, with
\begin{align*}
	\normH{(\bI - \dic{\cp{P}})\kappa(\bx,\cdot)}^2
	&= \normH{\kappa(\bx,\cdot)}^2 - \normH{\dic{\cp{P}}\kappa(\bx,\cdot)}^2
\\
	&= \kappa(\bx,\bx) - \normH{\dic{\cp{P}}\kappa(\bx,\cdot)}^2,
\end{align*}
where $\bI$ is the identity operator. Therefore, the quadratic approximation error is
\begin{equation}\label{eq:Pythagorean}
	\normH{(\bI - \dic{\cp{P}})\kappa(\bx,\cdot)}^2
	=  \kappa(\bx,\bx) -
	 \max_\bbeta \frac{\sum_{j=1}^m \beta_j \kappa(\bx,\dx_j)}{\normH{\sum_{j=1}^m \beta_j \kappa(\dx_j,\cdot)}}.
\end{equation}
In the following, we investigate this expression in order to derive proper bounds on the approximation error, either when the sample is discarded or when it already belongs to the dictionary.

\subsection{Approximation error of discarded samples}
\label{sec:approx.error.discard}

When the sample $\bx_t$ is discarded, we propose to upper bound the quadratic approximation error \eqref{eq:Pythagorean} with 
\begin{equation}\label{eq:error.discard}
	\normH{(\bI - \dic{\cp{P}})\kappa(\bx_t,\cdot)}^2
%	=  \kappa(\bx,\bx) - \max_\bbeta \frac{\sum_{j=1}^m \beta_j \kappa(\bx,\dx_j)}{\normH{\sum_{j=1}^m \beta_j \kappa(\dx_j,\cdot)}}
	 \leq \kappa(\bx_t,\bx_t) -
	 \max_j \frac{ |\kappa(\bx_t,\dx_j)|}{\sqrt{\kappa(\dx_j,\dx_j)}},
\end{equation}
where the inequality corresponds to the special choice of the coefficients, with $\beta_1= \ldots = \beta_m=0$ except for $\beta_j = \mathrm{sign}(\kappa(\bx_t,\dx_j))$. Next, we show that the quotient ${|\kappa(\bx_t,\dx_j)|}/{\sqrt{\kappa(\dx_j,\dx_j)}}$ in the above expression is bounded, with a lower bound that depends on the threshold of the investigated sparsification criterion. For this purpose, we examine separately the distance (Theorem~\ref{th:error.discard.dist}), the coherence (Theorem~\ref{th:error.discard.coher}), and the Babel (Theorem~\ref{th:error.discard.Babel}) criteria.

\medskip

\begin{theorem}[Discarding error for the distance criterion]\label{th:error.discard.dist}
	Let $\bx_t$ be a sample not satisfying the distance condition \eqref{eq:dist.crit1} for some given threshold $\delta$. The quadratic error of approximating $\kappa(\bx_t,\cdot)$ with a linear combination of atoms from the resulting dictionary is upper-bounded by 
	$$\delta^2 
	\qquad \textrm{and} \qquad
	\kappa(\bx_t,\bx_t) - \sqrt{\kappa(\bx_t,\bx_t) - \delta^2}.$$
	The latter upper bound is sharper when $\delta^2 > \kappa(\bx_t,\bx_t) - 1$. This is the case when dealing with unit-norm atoms, where we get the upper bound $1 - \sqrt{1 - \delta^2}$.
\end{theorem}

\begin{proof}
Firstly, one can easily derive the first expression of the upper bound, since
\begin{align*}
    \min_{\xi_1\cdots\xi_m} & \normHbig{\kappa(\bx_t,\cdot)  - \sum_{i=1}^m \xi_i \,\kappa(\dx_i,\cdot)}^2 
\\ & \qquad\qquad\qquad \leq 
    \min_{j=1\cdots m} \min_{\xi_j}\normHbig{\kappa(\bx_t,\cdot) - \xi_j \, \kappa(\dx_j,\cdot)}^2 
\\ & \qquad\qquad\qquad  < \delta^2,
\end{align*}
where the first inequality follows from the special case when all $\xi_i$ are null except for a single one, and the second inequality is due to the violation of \eqref{eq:dist.crit}.

Secondly, the second expression of the upper bound is a bit more trickier. The approximation error is given by the norm of the residual of the projection of $\kappa(\bx_t,\cdot)$ onto the subspace spanned by the dictionary atoms, namely as given in \eqref{eq:error.discard}. %Next, we show that the fraction ${|\kappa(\bx_t,\dx_j)|}/{\sqrt{\kappa(\dx_j,\dx_j)}}$ in that expression is bounded. To this end, we use the fact that 
Since $\bx_t$ does not satisfy the condition \eqref{eq:dist.crit}-\eqref{eq:dist.crit1}, %which means that there exists at least one atom $\dx_j$ of the dictionary such that
we have
\begin{equation*}%\label{eq:distbis}
    \min_{j=1\ldots m} \left( \kappa(\bx_t,\bx_t) - \frac{\kappa(\bx_t,\dx_j)^2}{\kappa(\dx_j,\dx_j)}  \right)
    < \delta^2,
\end{equation*}
and as a consequence, since $\kappa(\bx_t,\bx_t) \geq \delta^2$, we can easily show that
\begin{equation*}%\label{eq:distbis} 
    \sqrt{\kappa(\bx_t,\bx_t) - \delta^2}
    < \max_{j=1\ldots m} \frac{|\kappa(\bx_t,\dx_j)|}{\sqrt{\kappa(\dx_j,\dx_j)}}.
\end{equation*}
By injecting this inequality in \eqref{eq:error.discard}, we get the second expression of the upper bound with
\begin{equation*}
	\normHbig{\big(\bI - \dic{\cp{P}} \big)\kappa(\bx_t,\cdot)}^2
	< \kappa(\bx_t,\bx_t) - \sqrt{\kappa(\bx_t,\bx_t) - \delta^2}.
\end{equation*}	

Finally, we compare these two expressions. It is easy to see that $\kappa(\bx_t,\bx_t) - \sqrt{\kappa(\bx_t,\bx_t) - \delta^2}$ is sharper than $\delta^2$ when
$$\kappa(\bx_t,\bx_t) - \sqrt{\kappa(\bx_t,\bx_t) - \delta^2}
	< \delta^2,$$
namely when $\kappa(\bx_t,\bx_t) - \delta^2 < \sqrt{\kappa(\bx_t,\bx_t) - \delta^2}$. This condition, of the form $u^2 < u$, is satisfied when $u<1$, namely $\kappa(\bx_t,\bx_t) - \delta^2 < 1$.
\end{proof}

\begin{theorem}[Discarding error for the coherence criterion]\label{th:error.discard.coher}
	Let $\bx_t$ be a sample not satisfying the coherence condition \eqref{eq:coher.crit} for some given threshold $\gamma$. The quadratic error of approximating $\kappa(\bx_t,\cdot)$ with a linear combination of atoms from the resulting dictionary is upper-bounded by 
	$$\kappa(\bx_t,\bx_t)-\gamma \sqrt{\kappa(\bx_t,\bx_t)}.$$ 
	In the particular case of unit-norm atoms, we get $1 - \gamma$.
\end{theorem}

\begin{proof}
The unfulfilled coherence condition \eqref{eq:coher.crit}, namely
\begin{equation*}
	\max_{j=1\cdots m} \frac{|\kappa(\bx_t,\dx_j)|}
	{\sqrt{\kappa(\bx_t,\bx_t) \, \kappa(\dx_j,\dx_j)}} > \gamma,
\end{equation*}
can be written in the equivalent form
\begin{equation*}
	\max_{j=1\cdots m} \frac{|\kappa(\bx_t,\dx_j)|}
	{\sqrt{\kappa(\dx_j,\dx_j)}} > \gamma \sqrt{\kappa(\bx_t,\bx_t)}.
\end{equation*}
By injecting this inequality in \eqref{eq:error.discard}, we get an upper bound on the approximation error as follows:
\begin{align*}
	\normH{(\bI - \dic{\cp{P}})\kappa(\bx_t,\cdot)}^2
	%&\leq \kappa(\bx_t,\bx_t) -
	 %\max_{j=1\cdots m} \frac{ |\kappa(\bx_t,\dx_j)|}{\sqrt{\kappa(\dx_j,\dx_j)}}
	 %\\
	&< \kappa(\bx_t,\bx_t) - \gamma \, \sqrt{\kappa(\bx_t,\bx_t)},
\end{align*}
which concludes the proof.
\end{proof}

\begin{theorem}[Discarding error for the Babel criterion]\label{th:error.discard.Babel}
	Let $\bx_t$ be a sample not satisfying the Babel condition \eqref{eq:babel.crit} for some given threshold $\gamma$. The quadratic error of approximating $\kappa(\bx_t,\cdot)$ with a linear combination of atoms from the resulting dictionary is upper-bounded by
	$$\kappa(\bx_t,\bx_t) - \frac{\gamma}{\sqrt{m(R^2+\gamma)}},$$ 
	which becomes $1 - \frac{\gamma}{\sqrt{m(1+\gamma)}}$ for unit-norm atoms.
\end{theorem}

\begin{proof}
To prove this result, we use the quadratic approximation error given in expression \eqref{eq:Pythagorean} where, for the particular case of $\beta_j = \mathrm{sign}(\kappa(\bx_t,\dx_j))$, we get
\begin{align*}%\label{eq:Pythagorean}
	\normH{(\bI - \dic{\cp{P}})\kappa(\bx_t,\cdot)}^2
	&=  \kappa(\bx_t,\bx_t) -
	 \max_\bbeta \frac{\sum_{j=1}^m \beta_j \kappa(\bx_t,\dx_j)}{\normH{\sum_{j=1}^m \beta_j \kappa(\dx_j,\cdot)}}
\\	&\leq \kappa(\bx_t,\bx_t) -
	  \frac{\sum_{j=1}^m |\kappa(\bx_t,\dx_j)|}{(\bbeta^\top \! \dK \bbeta)^{\frac12}}
\end{align*}
The above numerator is bounded since the Babel condition \eqref{eq:babel.crit} is not satisfied, namely $\sum_{j=1}^m |\kappa(\bx_t, \dx_j)| > \gamma$. The above denominator is bounded thanks to the min-max theorem (\ie, the Rayleigh-Ritz quotient), with
\begin{equation*}
	\bbeta^\top \! \dK \bbeta \leq \dic{\lambda}_1 \| \bbeta \|^2.% = \dic{\lambda}_1 \, m,
\end{equation*}
It turns out that the this upper bound is equal to $m (R^2+\gamma)$. To show this, on the one hand, we have $\sum_{j=1}^m \beta_j^2 = \sum_{j=1}^m |\mathrm{sign} (\kappa(\bx_t,\dx_j))|^2 = m$ due to the aforementioned particular choice of the $\beta_j$. On the other hand, the eigenvalues of a Gram matrix $\dK$ associated to a $\gamma$-Babel dictionary are upper-bounded by $R^2+\gamma$, as given in the Appendix; see \cite[Theorem~5]{14.sparse.eigen} for more details. The combination of all these results concludes the proof.
\end{proof}

\subsection{Approximation error of an atom from the dictionary}\label{sec:approx.atom}

In this section, we study the approximation of an atom of a dictionary with a linear combination of its other atoms. We provide a lower bound on the approximation error for each sparsification criterion.

Let $\kappa(\dx_i,\cdot)$ be an atom of the dictionary, and consider its projection onto the subspace spanned by the other $m-1$ atoms. By following the same derivations as in the beginning of Section~\ref{sec:approx.error}, we have 
\begin{equation*}%\label{eq:Pythagorean}
	\normH{(\bI - \dic{\cp{P}})\kappa(\dx_i,\cdot)}^2
	=  \kappa(\dx_i,\dx_i) -
	 \max_\bbeta \frac{\sum_{j=1,j \neq i}^m \beta_j \kappa(\dx_i,\dx_j)}{\normH{\sum_{j=1,j \neq i}^m \beta_j \kappa(\dx_j,\cdot)}}.
\end{equation*}
On the one hand, the numerator in the above expression is upper-bounded, since we have from the Cauchy-Schwarz inequality:
\begin{equation*}
	\Biggr( \mathop{\sum_{j=1}^m}_{j \neq i} \beta_j \kappa(\dx_i,\dx_j) \Biggr)^2
	 \leq \mathop{\sum_{j=1}^m}_{j \neq i} \beta_j^2
		~ \mathop{\sum_{j=1}^m}_{j \neq i} \left| \kappa(\dx_i,\dx_j) \right|^2.
\end{equation*}
On the other hand, the denominator has a lower bound, since we have:
 \begin{equation*}
	\normHBig{\mathop{\sum_{j=1}^m}_{j \neq i} \beta_j \kappa(\dx_j,\cdot)}^2
	= \bbeta^\top \! \dK_{\!_{\setminus\!\{\!i\!\}}\!} \bbeta \geq \dic{\lambda}_{_{\setminus\!\{\!i\!\}}\!m-1} \|\bbeta \|^2,
	% \geq \dic{\lambda}_{m-1} \|\bbeta \|^2 ,
\end{equation*}
where $\dK_{\!_{\setminus\!\{\!i\!\}}\!}$ is the $(m-1)$-by-$(m-1)$ submatrix of the Gram matrix $\dK$ obtained by removing its $i$-th row and its $i$-th column, \ie, the entries associated to $\dx_i$, and $\dic{\lambda}_{_{\setminus\!\{\!i\!\}}\!m-1}$ is its smallest eigenvalue. By combining these two inequalities, we get
%\begin{equation}\label{eq:approx.error.L2}
%	\normH{(\bI - \dic{\cp{P}})\kappa(\dx_i,\cdot)}^2
%	\geq \kappa(\dx_i,\dx_i) -
%	 \frac{1}{\sqrt{\dic{\lambda}_{_{\setminus\!\{\!i\!\}}\!m-1}}}
%	 \Bigg( \mathop{\sum_{j=1}^m}_{j \neq i} | \kappa(\dx_i,\dx_j) |^2 \Bigg)^{\frac12} \!\!\!,
%\end{equation}
\begin{equation}\label{eq:approx.error.L2}
	\normH{(\bI - \dic{\cp{P}})\kappa(\dx_i,\cdot)}^2
	\geq \kappa(\dx_i,\dx_i) -
	 \sqrt{\frac{1}{\dic{\lambda}_{_{\setminus\!\{\!i\!\}}\!m-1}}
	 \mathop{\sum_{j=1}^m}_{j \neq i} | \kappa(\dx_i,\dx_j) |^2},
\end{equation}
This expression is investigated in the following for the distance (Theorem~\ref{th:error.atom.dist}), the coherence (Theorem~\ref{th:error.atom.coh}), and the Babel (Theorem~\ref{sec:approx.atom.babel}) criteria. For each sparsification criterion, the above lower bound is written by using the corresponding summation expression and the appropriate lower bound on the eigenvalues, as derived in \cite[Section~IV]{14.sparse.eigen} and put in a nutshell in the Appendix.% and the summation term.

\begin{theorem}[Acceptance error for the distance criterion]\label{th:error.atom.dist}
	For a dictionary resulting from the distance criterion for some given threshold $\delta$, the quadratic error of approximating any atom $\kappa(\dx_i,\cdot)$ with a linear combination of the other atoms is lower-bounded by 
\begin{equation*}
	\kappa(\dx_i,\dx_i) - \sqrt{\frac{\big(\kappa(\dx_i,\dx_i) - \delta^2\big) (m-1)R^2}{r^2 - (m-2)R\sqrt{R^2-\delta^2}}}.
\end{equation*}
For unit-norm atoms, we get a lower bound for all atoms, with 
$$1 - \sqrt{\frac{(m-1)(1-\delta^2)}{1 - (m-2)\sqrt{1-\delta^2}}}.$$
\end{theorem}

\begin{proof}
The proof is split in two parts, by investigating expression \eqref{eq:approx.error.L2}. Firstly, the summation term is upper-bounded since, from \eqref{eq:dist}, we have that any pair $(\dx_i,\dx_j)$ satisfies
\begin{align*}%\label{eq:dist.suminnerproduct}
	\mathop{\sum_{j=1}^m}_{j \neq i} |\kappa(\dx_i,\dx_j)|^2
& \leq \mathop{\sum_{j=1}^m}_{j \neq i} \kappa(\dx_j,\dx_j) \, \big( \kappa(\dx_i,\dx_i) - \delta^2 \big)
\\	& = \big(\kappa(\dx_i,\dx_i) - \delta^2\big) \mathop{\sum_{j=1}^m}_{j \neq i} \kappa(\dx_j,\dx_j)
\\	& = \big(\kappa(\dx_i,\dx_i) - \delta^2\big) (m-1)\mathop{\max_{j=1\cdots m}}_{j \neq i} \kappa(\dx_j,\dx_j)
\\	& = \big(\kappa(\dx_i,\dx_i) - \delta^2\big) (m-1)R^2.
\end{align*}
Secondly, the eigenvalue in this expression is lower-bounded by $r^2 - (m-2)R\sqrt{R^2-\delta^2}$ for a $\delta$-distant dictionary of $m-1$ atoms, as shown in Lemma~\ref{th:eigen.dist} of the Appendix.
\end{proof}

\begin{theorem}[Acceptance error for the coherence criterion]\label{th:error.atom.coh}
	For a dictionary resulting from the coherence criterion for some given threshold $\gamma$, the quadratic error of approximating any atom $\kappa(\dx_i,\cdot)$ with a linear combination of the other atoms is lower-bounded by 
\begin{equation*}
	\kappa(\dx_i,\dx_i) - \sqrt{\frac{(m-1) \, \gamma^2 R^2 \kappa(\dx_i,\dx_i)}{r^2 - (m-2)\gamma R^2}}.
\end{equation*}
For unit-norm atoms, this bounds becomes independent of $\dx_i$, with 
$$1 - \sqrt{\frac{(m-1) \, \gamma^2}{1 - (m-2) \, \gamma}}.$$
\end{theorem}

\begin{proof}
The proof follows the same procedure as in the previous proof. On the one hand, we have 
\begin{align*}
	\mathop{\sum_{j=1}^m}_{j \neq i} | \kappa(\dx_i,\dx_j) |^2
	&\leq (m-1) \mathop{\max_{j=1\cdots m}}_{j \neq i} | \kappa(\dx_i,\dx_j) |^2
\\	&\leq (m-1) \, \gamma^2 \mathop{\max_{j=1\cdots m}}_{j \neq i} \kappa(\dx_i,\dx_i) \, \kappa(\dx_j,\dx_j)
\\	&\leq (m-1) \, \gamma^2 R^2 \kappa(\dx_i,\dx_i),
\end{align*}
where the second inequality follows from the coherence condition. On the other hand, we use the lower bound $r^2 - (m-2)\gamma R^2$ on the eigenvalues associated to a $\gamma$-coherent dictionary of $m-1$ atoms, as derived in Lemma~\ref{th:eigen.coher} of the Appendix. To complete the proof, we combine these results in \eqref{eq:approx.error.L2}.
\end{proof}

\begin{theorem}[Acceptance error for the Babel criterion]\label{sec:approx.atom.babel}
	For a dictionary resulting from the Babel criterion for some given threshold $\gamma$, the quadratic error of approximating any atom $\kappa(\dx_i,\cdot)$ with a linear combination of the other atoms is lower-bounded by 
\begin{equation*}
	\kappa(\dx_i,\dx_i) - \frac{\gamma}{\sqrt{r^2-\gamma}}.
\end{equation*}
For unit-norm atoms, we get the following lower bound for all atoms:
$$1 - \frac{\gamma}{\sqrt{1 - \gamma}}.$$
\end{theorem}

\begin{proof}
The proof is obtained by substituting the $\ell_2$-norm in \eqref{eq:approx.error.L2}, \ie, $\big({\sum_{j}%=1, \ldots, m}_{j \neq i} 
| \kappa(\dx_i,\dx_j) |^2}\big)^{\frac12}$, with an $\ell_1$-norm, since we have the relation $\|\boldsymbol{u}\|_2 \leq \|\boldsymbol{u}\|_1$. This yields
\begin{equation*}%\label{eq:approx.error.L1}
	\normH{(\bI - \dic{\cp{P}})\kappa(\dx_i,\cdot)}^2
	\geq \kappa(\dx_i,\dx_i) -
	 \frac{1}{\sqrt{\dic{\lambda}_{_{\setminus\!\{\!i\!\}}\!m-1}}}
	 \mathop{\sum_{j=1}^m}_{j \neq i} | \kappa(\dx_i,\dx_j) |.
\end{equation*}
The above summation term is upper-bounded by $\gamma$ thanks to the Babel definition in \eqref{eq:babelD}. Moreover, the above eigenvalue is lower-bounded by $r^2-\gamma$, as derived in Lemma~\ref{th:eigen.babel} of the Appendix. This concludes the proof.
\end{proof}
This theorem should be compared with the work of \cite[Theorem~1]{Fan2014}, where the authors propose a lower bound on the quadratic approximation error for unit-norm atoms, with $1 - {\gamma}$. It is easy to see that the lower bound given in Theorem~\ref{sec:approx.atom.babel} is tighter than the previously proposed bound, and extends the result to atoms that are not unit-norm.

%%%%%%%%%%%%%%%%%%%

\section{Approximation of a feature}\label{sec:feature}

In this section, we study the relevance of approximating any feature with its projection onto the subspace spanned by the atoms of a dictionary. An upper bound on the approximation error is derived in the following theorem for any sparsification criterion, while specific bounds in term of the threshold of each criterion are given in the following Theorem~\ref{th:error.feature.criteria}. Moreover, these results are explored in two particular kernel-based learning algorithms, with the empirical mean (see Section~\ref{sec:feature.mean}) and the principal axes (see Section~\ref{sec:feature.kpca}) as features to be estimated.

\begin{theorem}\label{th:error.feature}
Consider the approximation of some feature $\psi(\cdot) = \sum_{i=1}^{n} \alpha_i \, \kappa(\bx_i,\cdot)$ with a sparse solution given by projecting it onto the subspace spanned by the $m$ atoms of a given dictionary. The quadratic error of such approximation is upper-bounded by 
\begin{equation*}
	\left( n - m \right) \, \|\balpha\|^2 \, \epsilon^2,
\end{equation*}
where $\epsilon$ is an upper bound on the approximation of any $\kappa(\bx_i,\cdot)$ with a linear combination of atoms from the dictionary.
\end{theorem}

\begin{proof}
Let $\dic{\cp{P}}$ be the projection operator onto the subspace spanned by the atoms of the dictionary under scrutiny, approximating $\psi(\cdot) = \sum_{i=1}^{n} \alpha_i \kappa(\bx_i,\cdot)$ with $\dic{\psi}(\cdot) = \sum_{j=1}^{m} \dic{\alpha}_j \kappa(\dx_j,\cdot)$. The error of such approximation is 
\begin{align}\label{eq:error.feature.first}
\nonumber
    \normH{ (\bI - \dic{\cp{P}}) \psi(\cdot)}
&= \normH{\sum_{i=1}^n \alpha_i \, (\bI - \dic{\cp{P}}) \, \kappa(\bx_i,\cdot)} 
\\&\leq \sum_{i=1}^n |\alpha_i| \, \normH{(\bI - \dic{\cp{P}}) \, \kappa(\bx_i,\cdot)},
\end{align}
where the inequality is due to the generalized triangular inequality. By applying the Cauchy-Schwarz inequality, we get the quadratic approximation error
 \begin{equation}\label{eq:error.feature.Cauchy.Schwarz}
    \normH{ (\bI - \dic{\cp{P}}) \psi(\cdot)}^2
   \leq \sum_{i=1}^n \alpha_i^{2} ~ \sum_{i=1}^n \normH{(\bI - \dic{\cp{P}}) \, \kappa(\bx_i,\cdot)}^2.
\end{equation}
The first summation is the quadratic $\ell_2$-norm of the vector of coefficients, namely $\|\balpha\|^2$. For the second summation, we separate it in two terms, entries belonging to the dictionary and those discarded thanks to the used sparsification criterion. While the former do not contribute to the error, %$\normH{ (\bI - \dic{\cp{P}}) \kappa(\bx_i,\cdot)} = 0$ for any kernel function $\kappa(\bx_i,\cdot)$ of the dictionary
only the latter take part in the summation, namely the $n - m$ discarded samples where $m$ is the size of the dictionary. Let $\epsilon^2$ be an upper bound on the quadratic error of discarding samples, as given in Section~\ref{sec:approx.error.discard}. Then, we get
\begin{align*}
    \normH{ (\bI - \dic{\cp{P}}) \psi(\cdot)}^2
    %&\leq \|\balpha\|^2 \sum_{\substack{i=1 \\ \red \bx_i \not\in \D}}^n  \normH{ (\bI - \dic{\cp{P}}) \, \kappa(\bx_i,\cdot)}^2\\    
    &\leq \left( n - m \right) \|\balpha\|^2 ~\epsilon^2,
\end{align*}
which concludes the proof.
\end{proof}

By revisiting the upper bounds given in Section~\ref{sec:approx.error.discard} for each sparsification criterion, we can easily show the following results. Expressions for non-unit-norm atoms can be derived without difficulty from Theorems \ref{th:error.discard.dist}, \ref{th:error.discard.coher}, \ref{th:error.discard.Babel}.
\begin{theorem}\label{th:error.feature.criteria}
The upper bound given in Theorem~\ref{th:error.feature} can be specified for each sparsification criterion in terms of the used threshold. For unit-norm atoms, we have
\begin{itemize}
\item $( n - m) \|\balpha\|^2 (1 - \sqrt{1 - \delta^2})$ for the $\delta$-distant criterion.
\item $( n - m) \|\balpha\|^2 \delta^2$ for the $\delta$-approximate criterion.
\item $( n - m) \|\balpha\|^2 (1-\gamma)$ for the $\gamma$-coherent criterion.
\item $( n - m) \|\balpha\|^2 \big(1-{\gamma}/{\sqrt{m(1+\gamma)}}\big)$ for the $\gamma$-Babel criterion.
\end{itemize}
\end{theorem}

We explore next these results for two particular kernel-based learning algorithms, in order to clarify the relevance of these bounds.

\subsection{Approximation of the empirical mean}\label{sec:feature.mean}

The empirical mean is a fundamental feature of the set of sample, and its use is essential in many statistical methods. For instance, it is investigated in \cite{Jenssen13} for visualization and clustering of nonnegative data and in \cite{12.isit,12.eusipco.oneclass} for one-class classification with kernel-based methods. In the following, we study the relevance of approximating the empirical mean by its projection onto the subspace spanned by the atoms of a dictionary. Let $\psi(\cdot) = \frac1n \sum_{i=1}^{n} \kappa(\bx_i,\cdot)$ be the empirical mean, namely $\alpha_i = 1/n$ for any $i=1,2, \ldots, n$. From Theorem~\ref{th:error.feature}, we get 
\begin{equation}\label{eq:error.feature.mean.first}
	\normH{ (\bI - \dic{\cp{P}}) \psi(\cdot)}^2 \leq \left( 1 - \frac{m}{n} \right) \epsilon^2,
\end{equation}
where $\max_i \normH{ (\bI - \dic{\cp{P}}) \kappa(\bx_i,\cdot)} \leq \epsilon$. In the following, we give a sharper bound.

Indeed, we provide a sharper bound by relaxing the use of the Cauchy-Schwarz inequality in \eqref{eq:error.feature.Cauchy.Schwarz}, thanks to the fact that the coefficients $\alpha_i$ are constant, \ie, independent of $i$. As a consequence, we get by revisiting expression \eqref{eq:error.feature.first}:
\begin{eqnarray*}%\label{eq:error.feature.first}
    \normH{ (\bI - \dic{\cp{P}}) \psi(\cdot)}
  &\leq& \sum_{i=1}^n |\alpha_i| \normH{(\bI - \dic{\cp{P}}) \, \kappa(\bx_i,\cdot)}
\\  &=& \frac{1}{n} \sum_{i=1}^n \normH{(\bI - \dic{\cp{P}}) \, \kappa(\bx_i,\cdot)}
%\\  &=& \frac{1}{n} \sum_{\substack{i=1 \\ \bx_i \not\in \D}}^n  \normH{ (\bI - \dic{\cp{P}}) \, \kappa(\bx_i,\cdot)}
\\  &\leq& \frac1n \left( n - m \right) \epsilon,
%\\  &\leq& \left( 1 - \frac{m}{n} \right) \epsilon.
\end{eqnarray*}
where we have followed the same decomposition as in the proof of Theorem~\ref{th:error.feature}, with only the $n - m$ discarded samples contribute to the summation term. Therefore, the quadratic approximation error is upper-bounded as follows:
\begin{equation*}
	\normH{ (\bI - \dic{\cp{P}}) \psi(\cdot)}^2
	    \leq \left( 1 - \frac{m}{n} \right)^2 \epsilon^2.
\end{equation*}
This bound is sharper than the one in \eqref{eq:error.feature.mean.first} since $1 - \frac{m}{n}<1$.

By revisiting Theorem~\ref{th:error.feature.criteria} in the light of this result, the upper bound on the quadratic approximation error $\normH{ (\bI - \dic{\cp{P}}) \psi(\cdot)}^2$ can be described in terms of the threshold of each sparsification criterion, as follows:
\begin{itemize}
\item $\displaystyle\left(1 - \frac{m}{n} \right)^2 \big(1 - \sqrt{1 - \delta^2}\big)$ for the $\delta$-distant criterion.
\item $\displaystyle\left(1 - \frac{m}{n} \right)^2 \delta^2$ for the $\delta$-approximate criterion.
\item $\displaystyle\left(1 - \frac{m}{n} \right)^2 (1-\gamma)$ for the $\gamma$-coherent criterion.
\item $\displaystyle\left(1 - \frac{m}{n} \right)^2 \left(1-\frac{\gamma}{\sqrt{m(1+\gamma)}} \right)$ for the $\gamma$-Babel criterion.
\end{itemize}
These results generalize the work in \cite{12.ssp.one_class}, where only the case of the coherence criterion is studied. Expressions for dictionaries with atoms that are not unit-norm can be easily obtained thanks to Theorems \ref{th:error.discard.dist}, \ref{th:error.discard.coher}, \ref{th:error.discard.Babel}.

\subsection{Approximation of the most relevant principal axes}\label{sec:feature.kpca}

Any sparsification criterion can be seen as a dimensionality reduction technique, because it identifies a subspace by selecting relevant samples from the available ones. Since it is an unsupervised approach, it is natural to connect it with the kernel principal component analysis. For the sake of clarity, it is assumed that the data are centered in the feature space; see \cite{14.tpami.center} for connections to the uncentered case.

The principal component analysis (PCA) seeks the principal axes that capture the most of the data variance. The principal axes correspond to the eigenvectors associated to the largest eigenvalues of the covariance matrix. In its kernel-based counterpart, \ie, the kernel-PCA, the $k$-th principal axis takes the form $\psi_k(\cdot)=\sum_{i=1}^n \alpha_{i,k} \, \kappa(\bx_i,\cdot)$, where the coefficients $\alpha_{i,k}$ are the entries of the $k$-th eigenvector of the Gram matrix $\bK$. Moreover, to get unit-norm principal axes, the coefficients $\alpha_{i,k}$ are normalized such that $\sum_{i=1}^n \alpha_{i,k}^2=1/n\lambda_k$. In this expression, $\lambda_k$ is the $k$-th eigenvalue of the Gram matrix, also called principal value. In the following, we highlight the connections between the kernel-PCA and the online sparsification criteria.

\begin{theorem}\label{th:error.feature.kpca}
Let $\psi_k(\cdot)$ be the $k$-th principal axe of the kernel functions $\kappa(\bx_i,\cdot)$, for $i=1,2, \ldots, n$, associated to the eigenvalue $\lambda_k$ of the corresponding Gram matrix. Its approximation with a dictionary of $m$ kernel functions has a quadratic error that can be upper-bounded by 
\begin{equation*}
	\left( 1 - \frac{m}{n} \right) \frac{\epsilon^2}{\lambda_k},
\end{equation*}
where $\epsilon$ is an upper bound on the approximation of any $\kappa(\bx_i,\cdot)$ with a linear combination of atoms from the dictionary.
\end{theorem}

The proof of this theorem is straightforward, by substituting $\|\balpha\|^2$ with ${1}/{n \lambda_k}$ in Theorem~\ref{th:error.feature}. Theorem~\ref{th:error.feature.kpca} shows that, under the only condition that the used dictionary has an upper bound on the error of approximating each kernel function, the principal axes associated to the largest principal values have the smallest approximation errors. One can therefore say that the most relevant principal axes lie, with a small error, in the span of the sparse dictionary.

Moreover, we derive expressions for each sparsification criterion, as given next in terms of the used threshold:
\begin{itemize}
\item $\displaystyle\left(1 - \frac{m}{n} \right) \frac{1 - \sqrt{1 - \delta^2}}{\lambda_k}$ for the $\delta$-distant criterion.
\item $\displaystyle\left(1 - \frac{m}{n} \right) \frac{\delta^2}{\lambda_k}$ for the $\delta$-approximate criterion.
\item $\displaystyle\left(1 - \frac{m}{n} \right) \frac{1-\gamma}{\lambda_k}$ for the $\gamma$-coherent criterion.
\item $\displaystyle\left(1 - \frac{m}{n} \right) \frac{1}{\lambda_k} \Bigg(1-\frac{\gamma}{\sqrt{m(1+\gamma)}} \Bigg)$ for the $\gamma$-Babel criterion.
\end{itemize}
These results generalize previous work on the approximation and the coherence criteria, and provide tighter bounds than the ones previously known in the literature. Indeed, the upper bound ${\delta^2}/{\lambda_k}$ was derived for the approximation criterion in \cite[Theorem~3.3]{Eng04} and in \cite[Theorem~5]{12.tpami}, while the coherence criterion is studied in \cite[Proposition~5]{Hon07.isit} with the upper bound ${(1-\gamma)}/{\lambda_k}$. %It is easy to seen that the bounds proposed in this paper are tighter than these previously known bounds.

%------------------------------------------------------------------------------------
%------------------------------------------------------------------------------------
%------------------------------------------------------------------------------------
%
%\newpage
\bigskip

\section{Final remarks}\label{sec:final_remarks}

In this paper, we %connected the approximation criterion with computationally efficient sparsification criteria. We 
studied the approximation errors of any sample when dealing with the distance, the coherence, or the Babel criterion, revealing that these criteria are roughly based on an approximation process. By deriving an upper bound on the error of approximating a sample discarded from the dictionary, we explored that the atoms are ``sufficient'' to represent any sample. The dual condition, namely showing that each atom of the dictionary is ``necessary'', was also exhibited by providing a lower bound on the approximation of any atom of the dictionary with the other atoms. Moreover, beyond the analysis of a single sample, we extended these results to the estimation of any feature, by describing in detail two classes of features, the empirical mean (\ie, centroid) and the principal axes in kernel-PCA. 

This work did not devise any particular sparsification criterion. It provided a framework to study online sparsification criteria. We argued that these criteria behave essentially in an identical mechanism, and share many interesting and desirable properties. Without loss of generality, we considered the framework of kernel-based learning algorithms. It is worth noting that these machines are intimately connected with the Gaussian processes \cite{gpml}, where the approximation criterion was initially proposed \cite{Csato02}.

\bigskip
%------------------------------------------------------------------------------------
% Appendix
%------------------------------------------------------------------------------------
\appendix

%\section{Bounds on the eigenvalues}
\label{sec:eigen}

This appendix provides bounds on the eigenvalues of a Gram matrix associated to a sparse dictionary, for each of the sparsity measures investigated in this paper. For completeness, these bounds are put here in a nutshell ; see \cite[Section~IV]{14.sparse.eigen} for more details. A cornerstone of these results is the well-known Ger\v{s}gorin Discs Theorem \cite[Chapter~6]{Hor12}. Revisited here for a Gram matrix associated to a sparse dictionary, it states that any of its eigenvalues lies in the union of the $m$ discs, centered on each diagonal entry of $\dK$ with a radius given by the sum of the absolute values of the other $m-1$ entries from the same row. In other words, for each $\dic{\lambda}_i$, there exists at least one $j \in \{1, 2, \ldots, m\}$ such that
\begin{equation}\label{eq:gershgorin}
|\dic{\lambda}_i - \kappa(\dx_j,\dx_j)|  
		\leq \mathop{\sum_{j=1}^m}_{j \neq i} |\kappa(\dx_i,\dx_j)|.
\end{equation}
This theorem provides upper and lower bounds on the eigenvalues of a Gram matrix associated to a sparse dictionary, as described in the following for each sparsity measure.

%
%Next, we present the Cauchy interlacing theorem, revisited here for the Gram matrix; See \cite{PrincipalSubmatrices,interlacing98} for more details. Let $\dic{\lambda}_{_{\setminus\!\{\!j\!\}}\!1}, \dic{\lambda}_{_{\setminus\!\{\!j\!\}}\!2}, \ldots, \dic{\lambda}_{_{\setminus\!\{\!j\!\}}\!m-1}$ be the eigenvalues, given in non-increasing order, of the matrix$\dK_{\!_{\setminus\!\{\!j\!\}}}$, namely the $(m-1)$-by-$(m-1)$ submatrix of the Gram matrix of the dictionary obtained by removing the $j$-th element from it, \ie, the one associated to $\dx_j$. 
%\begin{lemma}[Cauchy interlacing theorem]\label{th:interlacing}
%	We have the following interlacing property, for any $i=1, 2, \ldots, m-1$:
%	\begin{equation*}
%		\dic{\lambda}_i \geq \dic{\lambda}_{_{\setminus\!\{\!j\!\}}\!i} \geq \dic{\lambda}_{i+1}.
%	\end{equation*}
%\end{lemma}
%This means that, considering the largest and smallest eigenvalues of both matrices, we have respectively: 
%$$\dic{\lambda}_1 \geq \dic{\lambda}_{_{\setminus\!\{\!j\!\}}\!1} \text{ and } \dic{\lambda}_{_{\setminus\!\{\!j\!\}}\!m-1} \geq \dic{\lambda}_m.$$

\begin{lemma}%[Eigenvalues associated to  ]
\label{th:eigen.dist}
The eigenvalues of a Gram matrix associated to a $\delta$-distant dictionary of $m$ atoms are bounded as follows: 
	\begin{align*}
		r^2 - (m-1)R\sqrt{R^2-\delta^2} 
		&\leq 
		\dic{\lambda}_m 
		\leq 
		\cdots 
\\		
		\cdots 
		&\leq
		\dic{\lambda}_1
		\leq 
		R^2 + (m-1)R\sqrt{R^2-\delta^2}.
	\end{align*}
%For unit-norm atoms, we get
%	\begin{equation*}
%		1 - (m-1)\sqrt{1-\delta^2} \leq 
%		\dic{\lambda}_m \leq 
%		\cdots \leq
%		\dic{\lambda}_1 \leq 
%		1+(m-1)\sqrt{1-\delta^2}.
%	\end{equation*}	
%	%where the numerator is $\|\dkappa(\dx_j) \|^2$.
\end{lemma}
\begin{proof}
From \eqref{eq:dist}, a $\delta$-distant dictionary satisfies
$$|\kappa(\dx_i,\dx_j)| \leq \sqrt{\kappa(\dx_j,\dx_j) \, \big( \kappa(\dx_i,\dx_i) - \delta^2 \big)},$$
for any $i=1,2,\ldots,m$, which yields
\begin{align*}%\label{eq:dist.suminnerproduct}
	\sum_{j} |\kappa(\dx_i,\dx_j)|
\nonumber& \leq \sum_j \sqrt{\kappa(\dx_j,\dx_j) \, \big( \kappa(\dx_i,\dx_i) - \delta^2 \big)}
\\	& = \sqrt{ \kappa(\dx_i,\dx_i) - \delta^2 } \sum_j \sqrt{\kappa(\dx_j,\dx_j)}.
\end{align*}
By substituting this relation in \eqref{eq:gershgorin}, we get that, for each eigenvalue $\dic{\lambda}_k$, there exists at least one $i$ such that
\begin{align*}
	|\dic{\lambda}_k - \kappa(\dx_i,\dx_i)|  
%	& \leq \mathop{\sum_{j=1}^m}_{j \neq i} |\kappa(\dx_i,\dx_j)| \\
	& \leq \sqrt{ \kappa(\dx_i,\dx_i) - \delta^2 } \mathop{\sum_{j=1}^m}_{j \neq i} \sqrt{\kappa(\dx_j,\dx_j)}.
\end{align*}
%By exploring these results, the proof is straightforward.
\end{proof}

\begin{lemma}\label{th:eigen.coher}
The eigenvalues of a Gram matrix associated to a $\gamma$-coherent dictionary of $m$ atoms are bounded as follows:
	\begin{equation*}
		r^2-(m-1) \gamma R^2 \leq 
		\dic{\lambda}_m \leq 
		\cdots \leq
		\dic{\lambda}_1 \leq 
		R^2+(m-1) \gamma R^2.
	\end{equation*}
%For unit-norm atoms, we get $$1-(m-1) \, \gamma \leq \dic{\lambda}_m \leq \cdots \leq \dic{\lambda}_1 \leq 1+(m-1) \, \gamma.$$
\end{lemma}
\begin{proof}
A $\gamma$-coherent dictionary satisfies
\begin{equation*}
	\mathop{\max_{j=1\cdots m}}_{j \neq i} 
	%\mathrm{coh}(\kappa(\dx_i,\cdot) , \kappa(\dx_j,\cdot))
	%\mathrm{coh}(\dx_i,\dx_j) 
	\frac{|{\kappa(\dx_i,\dx_j)}|} {\sqrt{\kappa(\dx_i,\dx_i) \, \kappa(\dx_j,\dx_j)}}
	\leq \gamma,
\end{equation*}
for any $i,j=1,2,\ldots,m$, which yields
\begin{align*}
	\mathop{\max_{j=1\cdots m}}_{j \neq i} |\kappa(\dx_i,\dx_j)|
	&\leq \gamma 
	\mathop{\max_{j=1\cdots m}}_{j \neq i} \sqrt{\kappa(\dx_i,\dx_i) \, \kappa(\dx_j,\dx_j)}
\\	&= \gamma \sqrt{\kappa(\dx_i,\dx_i)}
	\mathop{\max_{j=1\cdots m}}_{j \neq i} \sqrt{\kappa(\dx_j,\dx_j)}
\\	&\leq \gamma R \sqrt{\kappa(\dx_i,\dx_i)}.
\end{align*}
By injecting this expression in \eqref{eq:gershgorin}, we get
%Finally, the proof results from applying the Ger\v{s}gorin Discs Theorem, since%, for any eigenvalue $\dic{\lambda}_k$, there exists an $i$ such as
\begin{align*}
	%|\dic{\lambda}_k - \kappa(\dx_i,\dx_i)|  
	%	\leq 
	\mathop{\sum_{j=1}^m}_{j \neq i} |\kappa(\dx_i,\dx_j)|
		&\leq (m-1) \mathop{\max_{j=1\cdots m}}_{j \neq i} |\kappa(\dx_i,\dx_j)| 
\\		&\leq (m-1) \gamma R \sqrt{\kappa(\dx_i,\dx_i)}.
%\\		&\leq (m-1) \gamma R^2.
\end{align*}
Since $\kappa(\dx_i,\dx_i) \leq R^2$, this completes the proof.\\%This concludes the proof.
%\vskip -.25cm
\end{proof}

\begin{lemma}\label{th:eigen.babel}
The eigenvalues of a Gram matrix associated to a $\gamma$-Babel dictionary are bounded as follows:
	\begin{equation*}
		r^2-\gamma \leq 
		\dic{\lambda}_m \leq 
		\cdots \leq
		\dic{\lambda}_1 \leq 
		R^2+\gamma.
	\end{equation*}
% For unit-norm atoms, we get 
%$1-\gamma \leq \dic{\lambda}_m \leq \cdots \leq \dic{\lambda}_1 \leq 1+\gamma$.
\end{lemma}

\begin{proof}
The proof is straightforward from the Ger\v{s}gorin Discs Theorem, since expression \eqref{eq:gershgorin} becomes
\begin{equation*}
	|\dic{\lambda}_k - \kappa(\dx_i,\dx_i)|  
		\leq 
	\mathop{\sum_{j=1}^m}_{j \neq i} |\kappa(\dx_i,\dx_j)|
		\leq \gamma,
\end{equation*}
where the last inequality is due to the Babel measure.\\
\end{proof}

\bigskip
%\section*{Acknowledgment} \addcontentsline{toc}{section}{Acknowledgment}
%The author would like to thank C\'edric Richard for the helpful discussions.

%------------------------------------------------------------------------------------
% Bibliographie
%------------------------------------------------------------------------------------

%%\scriptsize
\bibliography{biblio_ph,bibdesk_Paul}
\bibliographystyle{ieeetr}

\begin{biography}[{}]%\includegraphics[width=1in,height=1.25in,clip,keepaspectratio]{Honeine_p}}]
{Paul Honeine} (M'07) was born in Beirut, Lebanon, on October 2, 1977. He received the Dipl.-Ing. degree in mechanical engineering in 2002 and the M.Sc. degree in industrial control in 2003, both from the Faculty of Engineering, the Lebanese University, Lebanon. In 2007, he received the Ph.D. degree in Systems Optimisation and Security from the University of Technology of Troyes, France, and was a Postdoctoral Research associate with the Systems Modeling and Dependability Laboratory, from 2007 to 2008. Since September 2008, he has been an assistant Professor at the University of Technology of Troyes, France. His research interests include nonstationary signal analysis and classification, nonlinear and statistical signal processing, sparse representations, machine learning. Of particular interest are applications to (wireless) sensor networks, biomedical signal processing, hyperspectral imagery and nonlinear adaptive system identification. He is the co-author (with C. Richard) of the 2009 Best Paper Award at the IEEE Workshop on Machine Learning for Signal Processing. Over the past 5 years, he has published more than 100 peer-reviewed papers. 
\end{biography}

%Specialist in machine learning and data mining with over 10 years of experience. Exceptional skills in nonlinear and statistical signal processing, with particular interest in (wireless) sensor networks, biomedical signal processing, hyperspectral imagery and nonlinear adaptive system identification. Outstanding skills rewarded over the past five years, including best paper awards and more than 100 peer-reviewed papers.

\end{document}